\documentclass[conference]{IEEEtran}
\IEEEoverridecommandlockouts

\usepackage{cite}
\usepackage{amsmath,amssymb,amsfonts}
\usepackage{graphicx}
\usepackage{textcomp}
\usepackage[table]{xcolor}
\usepackage{booktabs}
\usepackage{array}
\usepackage{multirow}
\usepackage{tikz}
\usetikzlibrary{positioning,arrows.meta,calc,fit,backgrounds}
\usepackage{algorithm}
\usepackage{algpseudocode}
\PassOptionsToPackage{hyphens}{url}
\usepackage[hidelinks]{hyperref}
\newcommand{\repourl}{{\def\UrlFont{\ttfamily\mdseries\footnotesize}%
\url{https://github.com/BennyLinntu/PACT-Pairwise-Anchored-Calibrated-Tuning-for-Single-Token-Typed-Decisions}}}

\definecolor{cBase}{HTML}{8D96A0}
\definecolor{cNimble}{HTML}{EE8A2E}
\definecolor{cCEonly}{HTML}{8C62C4}
\definecolor{cPACT}{HTML}{0E8C86}
\definecolor{cCE}{HTML}{4A4F55}
\definecolor{cCF}{HTML}{3B7DD8}
\definecolor{cPC}{HTML}{EE8A2E}
\definecolor{cNR}{HTML}{E0584C}
\definecolor{cEMD}{HTML}{8C62C4}

\definecolor{hdr}{HTML}{1F3A5F}
\definecolor{subhdr}{HTML}{E4ECF6}
\definecolor{rowgray}{HTML}{F2F3F5}
\definecolor{pactrow}{HTML}{DDF1EE}
\definecolor{bestc}{HTML}{B9E4DC}
\definecolor{sigcol}{HTML}{FFE3A8}
\definecolor{hurtB}{HTML}{F3AFA7}
\definecolor{hurtA}{HTML}{FADAD5}
\definecolor{helpA}{HTML}{D6EFDD}
\definecolor{helpB}{HTML}{A6DBB6}
\definecolor{vGood}{HTML}{1E8E5A}
\definecolor{vPart}{HTML}{3B7DD8}
\definecolor{vMixed}{HTML}{D98A1C}
\newcommand{\hd}[1]{\textcolor{white}{\textbf{#1}}}
\newcommand{\best}[1]{\cellcolor{bestc}\textbf{#1}}
\newcommand{\sigc}[1]{\cellcolor{sigcol}\textbf{#1}}
\newcommand{\gainc}[1]{#1}
\newcommand{\worse}[1]{\textcolor{cNR}{\textbf{#1}}}
\DeclareRobustCommand{\swatch}[1]{\textcolor{#1}{\rule[-0.5pt]{6.5pt}{6.5pt}}}
\newcommand{\verdict}[2]{{\setlength{\fboxsep}{1.5pt}\colorbox{#1}{\textcolor{white}{\scriptsize\textbf{\strut#2}}}}}

\newcommand{\Lce}{\mathcal{L}_{\mathrm{CE}}}
\newcommand{\Lcf}{\mathcal{L}_{\mathrm{CF}}}
\newcommand{\Lpc}{\mathcal{L}_{\mathrm{PC}}}
\newcommand{\Lnr}{\mathcal{L}_{\mathrm{NR}}}
\newcommand{\Lemd}{\mathcal{L}_{\mathrm{EMD}}}
\newcommand{\JSD}{\mathrm{JSD}}
\newcommand{\softplus}{\mathrm{softplus}}
\newcommand{\run}[1]{#1}
\newtheorem{proposition}{Proposition}

\begin{document}

\title{PACT: Pairwise-Anchored Calibrated Tuning\\
for Single-Token Typed Decisions}

\author{\IEEEauthorblockN{Yida Lin}
\IEEEauthorblockA{Victoria University of Wellington\\
Wellington, New Zealand\\
linyida@myvuw.ac.nz}}

\maketitle

\begin{abstract}
Single-token typed-decision models answer a schema question by reading the
logits of a few one-letter answer codes at a single position: they are fast
and return a probability for every allowed answer, but they are trained with
plain cross-entropy that ignores most of the structure in their training data.
We study such a model whose data is curated as \emph{contrastive pairs}---two
contexts that differ in one edited fact that flips the answer---each carrying a
machine-checked certificate that deleting the decisive sentence makes the fact
unknown. We propose \textbf{PACT}, which turns this structure into four
training terms that need no new annotation: a difference-in-differences margin
over each pair that is invariant to any shared logit offset, a
permutation-consistency term against answer-code position bias, an
evidence-necessity term on certificate-verified ablated contexts, and an
ordinal transport cost for rubric fields, plus a three-parameter contextual
temperature. On a frozen 324-item holdout with three seeds, PACT matches the
published recipe in accuracy (84.6 vs.\ 85.2\%; McNemar $p\ge0.50$ at every
seed) while giving the lowest position bias of all runs (answer flips under
relabelling 9.8 vs.\ 13.8\%) and the lowest ordinal error on rubric fields
(MAE 0.232 vs.\ 0.311). Against a control with the same optimiser and schedule
but cross-entropy only, PACT is significantly more accurate at two of three
seeds, halves the seed-to-seed spread and lowers NLL by 26\%. Seed-matched
ablations and pre-specified falsification tests locate these gains precisely:
no single term raises raw accuracy, and the method's value lies in robustness
and stability rather than headline accuracy. Code, data splits, trained
adapters and all run records are available at \repourl.
\end{abstract}


\section{Introduction}
\label{sec:intro}

A \emph{typed decision} is a classification whose label space is supplied at
inference time by a schema: pick one of five outcomes, answer a true/false
question, or place a case on a five-level rubric. Such decisions can be served
efficiently by rendering the context and schema once, assigning each allowed
answer a one-token code, reading the logits of those codes at one position, and
normalising over them~\cite{rogge2026,nimble2026}. There is no generated JSON to
parse and no sampling, and every allowed answer receives a probability.

The training side of this interface is less settled. The open Nimble
recipe~\cite{nimble2026} fine-tunes a 9B-parameter base model with
LoRA~\cite{hu2022lora} on 2{,}676 curated examples using cross-entropy over the
candidate logits. Its curation, however, is far richer than its objective.
Every example belongs to a \emph{contrastive pair} produced by editing at most
eight words of one evidence sentence so that the reference answer flips, and
every pair carries a certificate that records the decisive fact and verifies
that removing either focus sentence leaves that fact \emph{unknown}. The
objective sees each example on its own: the pairing, the certificate and the
arbitrariness of the letter assigned to each answer never enter the loss.

This paper asks what a training objective looks like when it does use that
structure. Our premise is that the most informative supervision in contrastive
data lies not in either example but in the \emph{difference} between
them---and in the counterfactual absence of the evidence---neither of which
cross-entropy can express. We make three contributions.

\begin{enumerate}
\item \textbf{A structure-derived objective (PACT).} Four terms that use only
  structure already present in the data (Section~\ref{sec:method}): a
  difference-in-differences margin whose value is provably unchanged by any
  logit offset the pair shares (Proposition~\ref{prop:offset}); a
  Jensen--Shannon consistency term across answer-code relabellings; an
  evidence-necessity term on certificate-verified ablated contexts, which the
  released recipe discards; and a squared earth-mover cost for ordinal rubric
  fields. A three-parameter contextual temperature conditions calibration on
  the number of choices and the prompt length. Serving prompt, answer codes
  and adapter format are unchanged, so PACT is a drop-in replacement.
\item \textbf{A controlled, falsifiable evaluation.} We reproduce the published
  recipe inside one harness, add a control that shares PACT's optimiser but
  not its loss terms, run every leave-one-out ablation, state in advance which
  observation would falsify each claim (Section~\ref{sec:hypotheses}), and report
  paired significance tests rather than leaderboard deltas.
\item \textbf{An honest account of what the terms buy.} PACT does not raise
  accuracy over the published recipe. It does produce the lowest position bias
  and ordinal error of all runs, stabilises an aggressive optimiser that on its
  own is significantly worse, and makes calibration a cheap post-hoc step
  (Section~\ref{sec:results}); the ablations show which terms carry which effect
  and where the method does not help.
\end{enumerate}

Fig.~\ref{fig:overview} summarises the method.

\begin{figure*}[t]
\centering
\resizebox{\textwidth}{!}{%
\begin{tikzpicture}[
  font=\footnotesize, >={Stealth[length=4pt]},
  view/.style={draw=#1!80!black, fill=#1!10, rounded corners=3pt, align=left,
               text width=4.4cm, inner sep=4pt, minimum height=1.0cm, line width=0.6pt},
  loss/.style={draw=#1!85!black, fill=#1!10, rounded corners=3pt, align=left,
               text width=5.5cm, inner sep=4pt, minimum height=0.76cm, line width=0.6pt},
  lab/.style={font=\footnotesize\bfseries, text=black!75},
  arr/.style={->, line width=0.6pt, draw=black!55},
]
\node[lab] at (2.4,2.75) {(1) Views built from one curated pair};
\node[view=cCF] (xa) at (2.4,1.7)
  {\textbf{$x_a$ base:} context containing evidence sentence $e$; reference $y_a$};
\node[view=cCF] (xb) at (2.4,0.5)
  {\textbf{$x_b$ counterfactual:} $e$ edited to $e'$ ($\le$8 words); reference $y_b\neq y_a$};
\node[view=cPC] (xp) at (2.4,-0.7)
  {\textbf{$x_a^{\pi}$ relabelled:} same text, answer codes permuted (A,B,C $\to$ C,A,B)};
\node[view=cNR] (xo) at (2.4,-1.9)
  {\textbf{$x^{\circ}$ ablated:} focus sentence deleted; certificate says fact is \emph{unknown}};

\node[lab, text=cPACT!70!black] at (7.75,2.75) {(2) Shared scorer};
\node[draw=cPACT!80!black, fill=cPACT!8, rounded corners=4pt, line width=0.8pt,
      minimum width=3.4cm, minimum height=4.7cm] (llm) at (7.75,-0.1) {};
\node[align=center, text width=3.1cm] at (7.75,0.95)
  {\textbf{Base LLM + LoRA}\\[3pt]
   render context + schema,\\ read the $C$ code-token\\ logits at one position};
\node[draw=cPACT!80!black, fill=white, rounded corners=2pt, align=center,
      text width=2.8cm, inner sep=3pt] (z) at (7.75,-1.25)
  {canonical logits\\ $z(x)\in\mathbb{R}^{C}$\\[1pt]
   {\scriptsize code map $\kappa$ undone}};
\foreach \v in {xa,xb,xp,xo}{\draw[arr] (\v.east) -- (\v.east -| llm.west);}

\node[lab] at (13.6,2.75) {(3) Training objective};
\node[loss=cCE]  (lce)  at (13.6,1.95) {$\Lce$: \;$-\log p_y$ on $x_a$ and $x_b$};
\node[loss=cCF]  (lcf)  at (13.6,1.0)  {$\Lcf=\softplus(m-d)$, \;$d$: log-odds shift over the pair};
\node[loss=cPC]  (lpc)  at (13.6,0.05) {$\Lpc=\JSD\big(p(x_a)\,\|\,p(x_a^{\pi})\big)$};
\node[loss=cNR]  (lnr)  at (13.6,-0.9) {$\Lnr=\big(z_{y_a}(x^{\circ})-z_{y_b}(x^{\circ})\big)^{2}$};
\node[loss=cEMD] (lemd) at (13.6,-1.85) {$\Lemd$: squared CDF distance on Score fields};
\coordinate (fan) at ($(llm.east |- z.east)+(0.3,0)$);
\draw[line width=0.6pt, draw=black!55] (z.east) -- (fan);
\fill[black!55] (fan) circle (1.2pt);
\foreach \l in {lce,lcf,lpc,lnr,lemd}{\draw[arr] (fan) to[out=0,in=180] (\l.west);}

\node[lab] at (18.95,2.75) {(4) Serving};
\node[draw=cPACT!80!black, fill=cPACT!16, rounded corners=3pt, align=center,
      text width=3.45cm, inner sep=4pt, line width=0.7pt] (temp) at (18.95,1.05)
  {\textbf{Contextual temperature}\\[2pt]
   $T(x)=\softplus\big(a+b\log C$\\ $+\,c\log(L/1000)\big)$\\[2pt]
   {\scriptsize 3 parameters, fitted on\\ held-out training families}};
\node[draw=black!45, fill=white, rounded corners=3pt, align=center,
      text width=3.45cm, inner sep=4pt, line width=0.6pt] (serve) at (18.95,-1.35)
  {\textbf{Single-pass decision}\\[2pt] $p=\mathrm{softmax}\big(z/T(x)\big)$\\[1pt]
   {\scriptsize prompt, codes, adapter unchanged}};
\draw[arr] (temp.south) -- (serve.north);
\draw[arr, dashed] (llm.south) -- ++(0,-0.35) -| (serve.south)
  node[pos=0.27, below, font=\scriptsize, text=black!60] {after training, the same scorer is served};
\end{tikzpicture}}
\caption{\textbf{PACT at a glance.} (1) Each curated pair yields two labelled
views (base, counterfactual), a copy of the base with its answer codes
relabelled, and a certificate-verified ablated context with the decisive
sentence removed. (2) All views share one forward pass per view through the
LoRA-adapted model and are mapped to canonical answer order. (3) Four terms are
added to cross-entropy, each using structure the released recipe discards.
(4) A three-parameter temperature, fitted on training-side families, calibrates
the served single-pass distribution. The colours of the views and loss terms
are used consistently throughout the paper.}
\label{fig:overview}
\end{figure*}
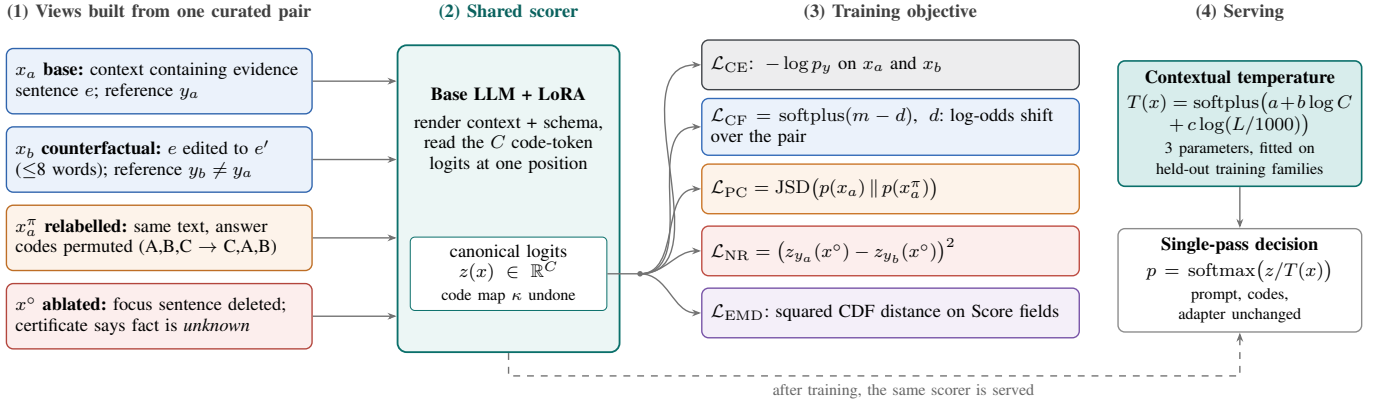

\section{Related Work}
\label{sec:related}

\textbf{Serving typed decisions.} Scoring candidate-token logits at a single
position instead of generating and parsing structured text underlies the
single-token decision interface described for this model
family~\cite{rogge2026,nimble2026}; MiniCheck~\cite{tang2024minicheck} applies
the same idea to one Boolean grounding decision. These works focus on the
serving interface and data; we keep both fixed and change only the objective.

\textbf{Counterfactual and contrastive supervision.} Counterfactually augmented
data~\cite{kaushik2020counterfactual} shows that minimal label-flipping edits
reduce reliance on spurious features, and contrast sets~\cite{gardner2020contrast}
use the same construction for evaluation. Both usually feed edited examples to
an unchanged objective. Teney et~al.~\cite{teney2020counterfactual} supervise
the input-gradient direction between counterfactual pairs; our margin instead
supervises the \emph{output} geometry of a pair and is exactly invariant to a
shared offset, in the spirit of difference-in-differences
estimation~\cite{angrist2009mostly} and pairwise ranking
losses~\cite{burges2005ranknet}. Evidence removal is standard for
\emph{evaluating} rationales~\cite{deyoung2020eraser}; we use certified
removals as a \emph{training} constraint.

\textbf{Option-order and label bias.} LLM multiple-choice answers depend on
option order and label identity~\cite{zheng2024selectors,pezeshkpour2024order,robinson2023mcsb},
and on label priors more generally~\cite{zhao2021calibrate}. Debiasing is
typically applied at inference by permuting and averaging or by
re-normalising. We additionally train against the bias with a bounded
Jensen--Shannon consistency loss~\cite{lin1991divergence}, analogous to the
consistency term of AugMix~\cite{hendrycks2020augmix}, which keeps single-pass
serving---the reason this model family exists---sound.

\textbf{Calibration.} Temperature scaling~\cite{guo2017calibration} is the
standard post-hoc method, with richer maps such as Dirichlet
calibration~\cite{kull2019dirichlet}; LLM confidence is known to be
miscalibrated in task-dependent ways~\cite{jiang2021know,kadavath2022know}. We
keep the temperature-scaling objective but let the temperature depend on two
observable prompt statistics, which preserves every argmax.

\textbf{Ordinal targets and efficient tuning.} Ordered labels benefit from
losses that respect the order~\cite{frank2001ordinal}; squared earth-mover
losses~\cite{hou2017emd} are one such choice, which we apply to rubric fields
whose expected level the serving path already reports. The adapter uses
LoRA~\cite{hu2022lora} with rank-stabilised scaling~\cite{kalajdzievski2023rslora}
and LoRA+ learning-rate ratios~\cite{hayou2024loraplus}; a
QLoRA~\cite{dettmers2023qlora} configuration is provided for 24\,GB GPUs.

\section{Problem Setting}
\label{sec:setting}

\subsection{Single-token typed decisions}
\label{sec:typed}
A field has allowed answers $v_1,\dots,v_C$ with $1\le C\le26$. The prompt
renders the context, the schema and the requested field, and shows answer
$v_k$ under a code letter. With $h$ the final-position hidden state and $t_k$
the token of the $k$-th code, the model produces $z_k=\langle w_{t_k},h\rangle$
and serves $p=\mathrm{softmax}(z_1,\dots,z_C)$. Fields are of three kinds:
\emph{Choice} (an enumeration), \emph{Boolean} (true/false) and \emph{Score}
(an ordered rubric with levels $0,\dots,C-1$). Two properties matter below.
First, the map from a \emph{meaning} to a \emph{letter} is arbitrary, so any
dependence of $p$ on the letter is error. Second, since one position is scored,
the decision is a linear read-out of one hidden state, which makes the
geometry of $z$---differences between candidate logits---the natural object to
supervise.

\subsection{Contrastive curation with certificates}
\label{sec:curation}
The training file contains 1{,}338 families of exactly two records: a base
context and a counterfactual that differs in one edited evidence sentence, with
the reference answer flipped and a byte-identical question and criteria block.
Each record carries a certificate naming the two sentences that jointly
determine the focus fact, their base and edited wording, and the verified
statement that with either focus sentence removed the focus fact becomes
\texttt{unknown}. The released recipe trains on the two records independently
and does not construct the removed-evidence contexts, correctly noting that
``missing evidence'' has no reference label. Our observation is that the
absence of a label is not the absence of information: the certificate
establishes \emph{indifference}, which is a constraint on the logit geometry
rather than a target. Fig.~\ref{fig:example} shows a real pair from the
training file and the three views PACT derives from it.

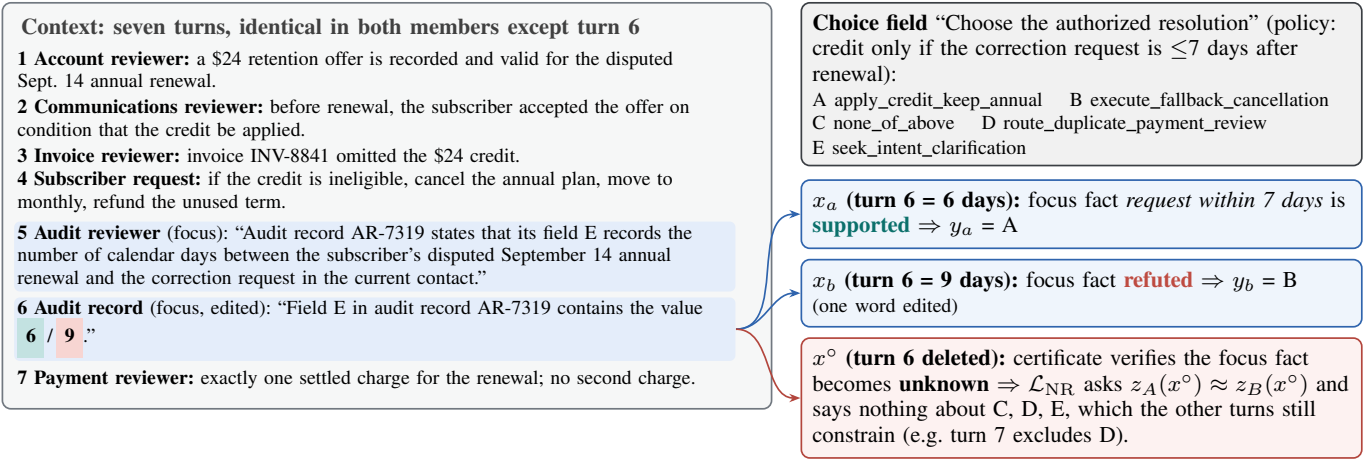
\begin{figure*}[t]
\centering
\resizebox{\textwidth}{!}{
\begin{tikzpicture}[
  font=\footnotesize, >={Stealth[length=4pt]},
  box/.style={draw=#1!80!black, fill=#1!8, rounded corners=3pt, align=left,
              inner sep=4pt, line width=0.6pt},
  turn/.style={align=left, text width=8.9cm, inner sep=1.2pt, font=\scriptsize},
  arr/.style={->, line width=0.6pt, draw=black!55},
]
\node[box=cBase, text width=9.3cm, minimum height=5.05cm] (ctx) at (0,0) {};
\node[anchor=north west, font=\footnotesize\bfseries, text=black!75]
  at ($(ctx.north west)+(0.12,-0.08)$) {Context: seven turns, identical in both members except turn 6};
\node[turn, anchor=north west] (t1) at ($(ctx.north west)+(0.15,-0.55)$)
  {\textbf{1 Account reviewer:} a \$24 retention offer is recorded and valid for the disputed Sept.~14 annual renewal.};
\node[turn, anchor=north west] (t2) at (t1.south west)
  {\textbf{2 Communications reviewer:} before renewal, the subscriber accepted the offer on condition that the credit be applied.};
\node[turn, anchor=north west] (t3) at (t2.south west)
  {\textbf{3 Invoice reviewer:} invoice INV-8841 omitted the \$24 credit.};
\node[turn, anchor=north west] (t4) at (t3.south west)
  {\textbf{4 Subscriber request:} if the credit is ineligible, cancel the annual plan, move to monthly, refund the unused term.};
\node[turn, anchor=north west, fill=cCF!14, rounded corners=2pt] (t5) at ($(t4.south west)+(0,-0.06)$)
  {\textbf{5 Audit reviewer} (focus): ``Audit record AR-7319 states that its field E records the number of
   calendar days between the subscriber's disputed September 14 annual renewal and the correction request
   in the current contact.''};
\node[turn, anchor=north west, fill=cCF!14, rounded corners=2pt] (t6) at ($(t5.south west)+(0,-0.06)$)
  {\textbf{6 Audit record} (focus, edited): ``Field E in audit record AR-7319 contains the value
   \colorbox{cPACT!25}{\strut\textbf{6}}\,/\,\colorbox{cNR!25}{\strut\textbf{9}}.''};
\node[turn, anchor=north west] (t7) at ($(t6.south west)+(0,-0.06)$)
  {\textbf{7 Payment reviewer:} exactly one settled charge for the renewal; no second charge.};

\node[box=cCE, text width=6.7cm, anchor=north west] (q) at ($(ctx.north east)+(0.35,0)$)
  {\textbf{Choice field} ``Choose the authorized resolution'' (policy: credit only if the
   correction request is $\le$7 days after renewal):\\[1pt]
   {\scriptsize A~apply\_credit\_keep\_annual \quad B~execute\_fallback\_cancellation\\
   C~none\_of\_above \quad D~route\_duplicate\_payment\_review\\
   E~seek\_intent\_clarification}};

\node[box=cCF, text width=6.7cm, anchor=north west] (va) at ($(q.south west)+(0,-0.16)$)
  {\textbf{$x_a$ (turn 6 = 6 days):} focus fact \emph{request within 7 days} is
   \textcolor{cPACT!80!black}{\textbf{supported}} $\Rightarrow$ $y_a$ = A};
\node[box=cCF, text width=6.7cm, anchor=north west] (vb) at ($(va.south west)+(0,-0.12)$)
  {\textbf{$x_b$ (turn 6 = 9 days):} focus fact \textcolor{cNR!85!black}{\textbf{refuted}}
   $\Rightarrow$ $y_b$ = B \quad {\scriptsize(one word edited)}};
\node[box=cNR, text width=6.7cm, anchor=north west] (vo) at ($(vb.south west)+(0,-0.12)$)
  {\textbf{$x^{\circ}$ (turn 6 deleted):} certificate verifies the focus fact becomes
   \textbf{unknown} $\Rightarrow$ $\Lnr$ asks $z_A(x^\circ)\approx z_B(x^\circ)$ and says
   nothing about C, D, E, which the other turns still constrain (e.g.\ turn 7 excludes D).};

\draw[arr, draw=cCF!80!black] (t6.east) to[out=0, in=180] (va.west);
\draw[arr, draw=cCF!80!black] (t6.east) to[out=0, in=180] (vb.west);
\draw[arr, draw=cNR!80!black] (t6.east) to[out=0, in=180] (vo.west);
\end{tikzpicture}}
\caption{\textbf{A curated contrastive pair from the training file} (family
\texttt{scale-diverse-013-002}, commerce domain; turns 1--4 and 7 abridged, focus
sentences verbatim). Changing one token of the second focus sentence (6$\to$9 days)
flips the certified focus fact and the reference answer from A to B; deleting that
sentence makes the fact \emph{unknown}. Cross-entropy sees $x_a$ and $x_b$ as two
unrelated examples; PACT uses the pair ($\Lcf$), a relabelled copy ($\Lpc$) and the
ablated context ($\Lnr$).}
\label{fig:example}
\end{figure*}

\subsection{Notation}
\label{sec:notation}
For an example $x$ with $C$ allowed answers, $z(x)\in\mathbb{R}^C$ denotes the
candidate logits in \emph{canonical} order (the order in which the schema
declares the answers) and $p(x)$ the corresponding distribution. A pair is
$(x_a,x_b)$ with references $y_a\neq y_b$ from the same canonical set;
$x^\circ$ is $x_a$ with one focus evidence sentence deleted; $\pi$ is a
permutation of code positions and $z^{\pi}$ the canonical logits obtained when
the prompt is rendered under $\pi$.

\section{Method}
\label{sec:method}

\subsection{Canonical choice space}
\label{sec:canon}
Every loss and metric is computed after mapping code positions back to
canonical answer indices. Per view, the encoder records a vector $\kappa$ with
$\kappa_k$ the canonical index of the answer shown at code position $k$; logits
are scattered from code space into canonical space and unused positions are
masked to $-\infty$. This makes the pair terms well defined when the two
members are rendered under different permutations, and makes ``the model
changed its answer when the letters changed'' a measurable event.

\subsection{Counterfactual difference-in-differences margin}
\label{sec:cf}
For a pair define
\begin{equation}
d(x_a,x_b)=\big[z_{y_a}(x_a)-z_{y_b}(x_a)\big]+\big[z_{y_b}(x_b)-z_{y_a}(x_b)\big],
\label{eq:did}
\end{equation}
i.e.\ the change of the $y_a{:}y_b$ log-odds between the two members, and
penalise $\Lcf=\softplus(m-d)$.

\begin{proposition}[Offset invariance]
\label{prop:offset}
For any $c\in\mathbb{R}^C$, replacing $z(x_a)$ by $z(x_a)+c$ and $z(x_b)$ by
$z(x_b)+c$ leaves $d$ unchanged. Moreover, if both members are correct under
the binary restriction to $\{y_a,y_b\}$, then $d>0$.
\end{proposition}
\begin{IEEEproof}
Each bracket in \eqref{eq:did} is a difference of two logits of the same input,
so $c_{y_a}-c_{y_b}$ enters the first bracket and $c_{y_b}-c_{y_a}$ the second;
they cancel. Correctness of $x_a$ ($x_b$) under the restriction means the first
(second) bracket is positive.
\end{IEEEproof}

A per-pair offset $c$ is exactly what a model learns when it acquires a prior
over outcomes from the topic, the policy text or the family's surface form.
$\Lcf$ therefore cannot be reduced by such a prior, only by making the logit
gap move in the right direction when the one edited fact moves. Cross-entropy
lacks this property: a model that assigns 0.51 to $y_a$ on the base and 0.49 on
the counterfactual satisfies both cross-entropy terms nearly as well as a model
with a decisive, evidence-driven gap. Because $d$ couples both members, it can
only be evaluated when both are in the same micro-batch (Section~\ref{sec:optim}). We
use $m=2$ nats and a softplus hinge so that gradients vanish smoothly once the
pair is separated.

\subsection{Permutation consistency}
\label{sec:pc}
For two code permutations $\pi,\pi'$ of the same example we penalise
\begin{equation}
\Lpc=\JSD\big(p^{\pi}\,\|\,p^{\pi'}\big)
=\tfrac12\mathrm{KL}(p^{\pi}\|\bar p)+\tfrac12\mathrm{KL}(p^{\pi'}\|\bar p),
\end{equation}
with $\bar p=\tfrac12(p^{\pi}+p^{\pi'})$. JSD is bounded and symmetric, so no
view is privileged as a teacher and a confidently wrong view cannot produce an
unbounded gradient; the term is label-free. At inference the same invariance
supports a $K$-view ensemble that averages canonical log-probabilities. We
treat the gain from $K{=}1$ to $K{=}4$ as a measure of the bias that training
failed to remove: a perfectly permutation-invariant model gains nothing.

\subsection{Evidence-necessity regularisation}
\label{sec:nr}
The certificate states that in $x^\circ$ the focus fact is unknown given all
remaining text, and the decision rules that separate $y_a$ from $y_b$ take the
focus fact as input. Hence the evidence in $x^\circ$ does not discriminate
$y_a$ from $y_b$---although it may still rule out other answers. We encode
exactly that and nothing more:
\begin{equation}
\Lnr=\Big(\mathrm{relu}\big(|z_{y_a}(x^\circ)-z_{y_b}(x^\circ)|-\varepsilon\big)\Big)^2 .
\end{equation}
The term is silent about $z_j(x^\circ)$ for $j\notin\{y_a,y_b\}$, about the
entropy of $p(x^\circ)$, and about which of $y_a,y_b$ is ``correct''---claims
the certificate does not support. To build $x^\circ$ we locate the focus
sentence by certificate path, then exact leaf match, then substring, then the
counterfactual wording, and delete the whole conversational turn if it carries
nothing else. On the released training file this succeeds for all 2{,}676
records, verified by checking that the sentence no longer occurs anywhere in
the serialised context. We use $\varepsilon=0$.

\subsection{Ordinal transport for rubric fields}
\label{sec:emd}
Score fields expose levels $0,\dots,C-1$ and the serving code reports the
expected level $\sum_j j\,p_j$. Cross-entropy is invariant to the ordering of
the levels, so nothing in training prefers a near miss to a far one. We add
the squared earth-mover distance between predicted and reference CDFs,
\begin{equation}
\Lemd=\sum_{j=0}^{C-1}\Big(\sum_{i\le j}p_i-\mathbf{1}[\,j\ge y\,]\Big)^2 ,
\end{equation}
which is differentiable and minimised by mass concentrated \emph{near} $y$
rather than merely \emph{on} $y$.

\subsection{Contextual temperature}
\label{sec:temp}
The released model ships without calibration. A single global temperature has
the wrong shape here: a two-choice Boolean and a six-choice enumeration do
not share a miscalibration factor, and neither do a 300-token and a
1{,}100-token prompt. We fit
\begin{equation}
T(x)=\softplus\big(a+b\log C+c\log(L/1000)\big),
\end{equation}
with $L$ the prompt length in tokens, by minimising NLL with L-BFGS on a
\emph{calibration split} of training-side families never used to fit the
adapter. With $b=c=0$ this is standard temperature
scaling~\cite{guo2017calibration}. Because $T(x)>0$ rescales each item's logits
by a positive constant, it never changes an argmax, so accuracy and pair
accuracy are untouched; it can, however, change the \emph{ranking} of
confidences across items and hence selective risk.

\subsection{Objective, batching and cost}
\label{sec:optim}
The full objective is
\begin{equation}
\mathcal{L}=\Lce+\lambda_{\mathrm{cf}}\Lcf+\lambda_{\mathrm{pc}}\Lpc
+\lambda_{\mathrm{nr}}\Lnr+\lambda_{\mathrm{emd}}\Lemd ,
\label{eq:objective}
\end{equation}
with $(\lambda_{\mathrm{cf}},\lambda_{\mathrm{pc}},\lambda_{\mathrm{nr}},\lambda_{\mathrm{emd}})=(0.5,0.5,0.2,0.3)$,
chosen for scale comparability and \emph{not} tuned on any evaluation data. The
regularisers ramp in linearly over the first 10\% of steps. Since
$\Lcf$ needs both pair members and $\Lpc$ needs both views, the unit of
batching is a \emph{group}: one family with all of its views
(Algorithm~\ref{alg:step}). Groups are shuffled, sorted by length inside
buckets of 64 groups (prompt lengths vary by more than 3$\times$, so padding
dominates waste) and cut into micro-batches. Each pair contributes two primary, two permuted
and one ablated view, i.e.\ 2.5$\times$ the rows of the baseline.

\begin{algorithm}[t]
\caption{One PACT optimiser step}
\label{alg:step}
\small
\begin{algorithmic}[1]
\Require groups $\mathcal{G}$ (pair $(x_a,x_b)$, permuted $x_a^{\pi},x_b^{\pi}$, ablated $x^\circ$); ramp $r_t\in[0,1]$
\For{each micro-batch $B\subset\mathcal{G}$ (length-bucketed)}
  \State one forward pass over all views in $B$; read code logits
  \State scatter to canonical order with $\kappa$; mask unused slots
  \State $\ell\gets\Lce+r_t\lambda_{\mathrm{emd}}\Lemd$ \Comment{labelled views}
  \State $\ell\gets\ell+r_t\big(\lambda_{\mathrm{cf}}\Lcf+\lambda_{\mathrm{pc}}\Lpc+\lambda_{\mathrm{nr}}\Lnr\big)$
  \State accumulate $\nabla\ell$
\EndFor
\State clip global norm to 1; AdamW step on LoRA $A$ ($\eta$) and $B$ ($8\eta$)
\end{algorithmic}
\end{algorithm}

\section{Experimental Setup}
\label{sec:setup}

\subsection{Data and splits}
Training uses the released training file (2{,}676 records, 1{,}338 pairs,
34 source families, 10 domains). Four source families are held out from
training: two form the \emph{selection split} used for checkpoint selection
and two the \emph{calibration split} used to fit $T(x)$ (Table~\ref{tab:data}).
The frozen \emph{holdout} (324 records, 162 pairs, 6 source families from 6
domains) is used only to report final numbers---never for selection,
calibration or early stopping. Family-level disjointness and the file hashes
of the dataset manifest are re-verified at load time. Holdout fields have
$C\in\{2,\dots,6\}$ allowed answers, and the longest rendered prompt has
1{,}098 tokens.

\begin{table}[t]
\centering
\caption{Data splits. Splits are disjoint at the level of source families; every record belongs
to a two-record contrastive pair. The holdout is never used for selection or calibration.}
\label{tab:data}
\setlength{\tabcolsep}{2.6pt}
\footnotesize
\rowcolors{2}{white}{rowgray}
\begin{tabular}{l c r r ccc l}
\toprule
\rowcolor{hdr}
\hd{Split} & \hd{Fam.} & \hd{Records} & \hd{Pairs} & \hd{Choice} & \hd{Bool.} & \hd{Score} & \hd{Role} \\
\midrule
Training & 30 & 2{,}334 & 1{,}167 & 766 & 796 & 772 & train \\
Selection & 2 & 116 & 58 & 26 & 32 & 58 & select \\
Calibration & 2 & 226 & 113 & 64 & 60 & 102 & fit $T$ \\
\rowcolor{pactrow}Holdout (frozen) & 6 & 324 & 162 & 146 & 114 & 64 & test only \\
\bottomrule
\end{tabular}
\end{table}

\subsection{Model and training}
The base model is Qwen3.5-9B~\cite{qwen35} at the revision pinned by the
published recipe, in BF16 with a 2{,}048-token limit. The adapter is LoRA of
rank 16 ($\alpha{=}32$, dropout 0.05) on all attention and MLP projections of
the language model---the modules the published recipe adapts---with rsLoRA
scaling. We train with AdamW~\cite{loshchilov2019adamw} at $\eta=5\times10^{-5}$,
LoRA+ ratio 8, cosine schedule with 10\% warm-up, two epochs (584 steps) and an
effective batch of 8 primary rows (4 pairs) per step, matching the published
recipe. The selection split is scored every 146 steps and the checkpoint with
the lowest selection NLL is kept. Jobs ran one per GPU on two 48\,GB GPUs.

\subsection{Compared runs}
Table~\ref{tab:runs} lists the runs. \run{Nimble recipe} reproduces the
published recipe inside our harness. \run{CE-only} shares \emph{everything}
with PACT except the four terms and the extra views, so PACT vs.\ \run{CE-only}
isolates the objective, while \run{Nimble recipe} vs.\ \run{CE-only} isolates
the optimisation changes. The three main arms use seeds 17, 18 and 19; each
leave-one-out ablation uses seed 17 and is compared against PACT at the
\emph{same} seed (Section~\ref{sec:ablation}).

\begin{table}[t]
\centering
\caption{Compared runs. All share data, prompt, adapter family and
checkpoint selection. ``Opt.\,A'': 1 epoch, linear schedule, plain LoRA
(published). ``Opt.\,B'': 2 epochs, cosine, rsLoRA + LoRA+.}
\label{tab:runs}
\setlength{\tabcolsep}{3pt}
\footnotesize
\rowcolors{2}{white}{rowgray}
\begin{tabular}{l l l c c}
\toprule
\rowcolor{hdr}
\hd{Run} & \hd{Objective} & \hd{Views$^{\ddagger}$} & \hd{Opt.} & \hd{Seeds} \\
\midrule
Base model & -- (adapter off) & -- & -- & -- \\
Nimble recipe & $\Lce$ & 2/0/0 & A & 3 \\
CE-only & $\Lce$ & 2/0/0 & B & 3 \\
\rowcolor{pactrow}\textbf{PACT} & Eq.~\eqref{eq:objective} & 2/2/1 & B & 3 \\
$-\Lcf$, $-\Lnr$, $-\Lemd$ & PACT minus one term & 2/2/1 & B & 1 \\
$-\Lpc$ & PACT minus JSD & 2/2$^{\ast}$/1 & B & 1 \\
$-$permuted views & PACT without $\Lpc$ & 2/0/1 & B & 1 \\
\bottomrule
\end{tabular}\\[2pt]
{\scriptsize\raggedright $^{\ddagger}$Primary / permuted / ablated views per pair.
$^{\ast}$Permuted views kept as plain cross-entropy augmentation.\par}
\end{table}

\subsection{Metrics and statistics}
We report accuracy; \textbf{pair accuracy}, the share of the 162 holdout pairs
with \emph{both} members correct; NLL; multiclass Brier
score~\cite{brier1950verification}; ECE over 15 equal-width
bins~\cite{naeini2015ece}; the area under the confidence-ranked risk--coverage
curve (AURC)~\cite{geifman2017selective}; expected-level MAE on Score fields;
and, for $K{=}4$ decoding, the mean pairwise \textbf{total variation (TV)}
between the canonical distributions of different code orderings and the
\textbf{flip rate}, the share of items whose argmax changes across orderings.
Because all runs are scored on the same items, we compare accuracies with the
exact McNemar test~\cite{mcnemar1947} and a 10{,}000-resample paired
bootstrap~\cite{efron1993bootstrap}, per seed.

\subsection{Pre-specified falsification criteria}
\label{sec:hypotheses}
The following were written before the runs, so the ablations read as tests
rather than a search over cells.
\begin{itemize}
\item[\textbf{H1}] \emph{Pair margin.} Removing $\Lcf$ lowers pair accuracy
  by more than it lowers accuracy.
\item[\textbf{H2}] \emph{Permutation consistency.} PACT has lower TV and flip
  rate than \run{CE-only} and gains less from $K{=}4$ ensembling.
\item[\textbf{H3}] \emph{Necessity.} Without $\Lnr$ the model is more
  confident on ablated contexts and has no better pair accuracy.
\item[\textbf{H4}] \emph{Ordinal transport.} Without $\Lemd$, Score MAE is
  higher at similar Score accuracy.
\item[\textbf{H5}] \emph{Calibration.} Holdout ECE falls after applying
  $T(x)$ without a loss of accuracy.
\end{itemize}

\section{Results}
\label{sec:results}

\begin{table*}[t]
\centering
\caption{Holdout results, single-pass decoding ($K{=}1$), mean\,$\pm$\,s.d.\ over seeds 17/18/19
(324 items, 162 pairs). Accuracy, pair accuracy, ECE and AURC in \%. ECE$_T$ is after the contextual
temperature $T(x)$ (not fitted for the base model). \best{Shaded}: best mean per column.}
\label{tab:main}
\setlength{\tabcolsep}{5.2pt}
\rowcolors{2}{white}{rowgray}
\begin{tabular}{l cc ccccc c}
\toprule
\rowcolor{hdr}
\hd{Run} & \hd{Acc.\,$\uparrow$} & \hd{Pair acc.\,$\uparrow$} & \hd{NLL\,$\downarrow$} & \hd{Brier\,$\downarrow$} &
\hd{ECE\,$\downarrow$} & \hd{ECE$_T$\,$\downarrow$} & \hd{AURC\,$\downarrow$} & \hd{Score MAE\,$\downarrow$} \\
\midrule
Base model & 66.4 & 35.8 & 1.094 & 0.549 & 24.6 & -- & -- & 0.571 \\
Nimble recipe & \best{85.2\,{\scriptsize$\pm$\,2.1}} & 71.0\,{\scriptsize$\pm$\,3.9} & \best{0.449\,{\scriptsize$\pm$\,0.110}} & \best{0.236\,{\scriptsize$\pm$\,0.042}} & \best{9.5\,{\scriptsize$\pm$\,3.4}} & \best{5.4\,{\scriptsize$\pm$\,0.8}} & \best{3.41\,{\scriptsize$\pm$\,0.55}} & 0.311\,{\scriptsize$\pm$\,0.030} \\
CE-only & 82.2\,{\scriptsize$\pm$\,6.0} & 70.2\,{\scriptsize$\pm$\,6.2} & 0.760\,{\scriptsize$\pm$\,0.296} & 0.305\,{\scriptsize$\pm$\,0.110} & 14.3\,{\scriptsize$\pm$\,6.4} & 10.3\,{\scriptsize$\pm$\,4.7} & 4.46\,{\scriptsize$\pm$\,2.94} & 0.322\,{\scriptsize$\pm$\,0.153} \\
\rowcolor{pactrow}\textbf{PACT (ours)} & 84.6\,{\scriptsize$\pm$\,2.8} & \best{73.3\,{\scriptsize$\pm$\,6.2}} & 0.560\,{\scriptsize$\pm$\,0.088} & 0.251\,{\scriptsize$\pm$\,0.026} & 10.8\,{\scriptsize$\pm$\,1.7} & 7.9\,{\scriptsize$\pm$\,1.1} & 3.95\,{\scriptsize$\pm$\,1.03} & \best{0.232\,{\scriptsize$\pm$\,0.120}} \\
\bottomrule
\end{tabular}
\end{table*}

\begin{figure*}[t]
\centering
\includegraphics[width=\textwidth]{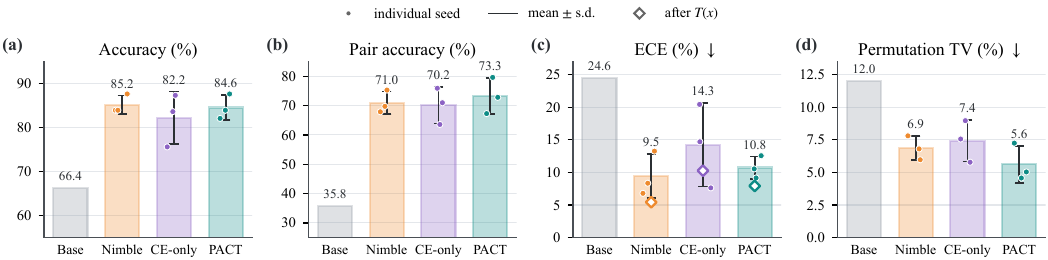}
\caption{Holdout comparison of the three multi-seed arms and the base model.
Bars: mean; dots: individual seeds; whiskers: $\pm$1 s.d. (a--c) single-pass
decoding; in (c) hollow diamonds give ECE after the contextual temperature.
(d) Permutation TV under $K{=}4$ relabellings (lower is less position bias).}
\label{fig:main}
\end{figure*}

\begin{figure*}[t]
\centering
\includegraphics[width=\textwidth]{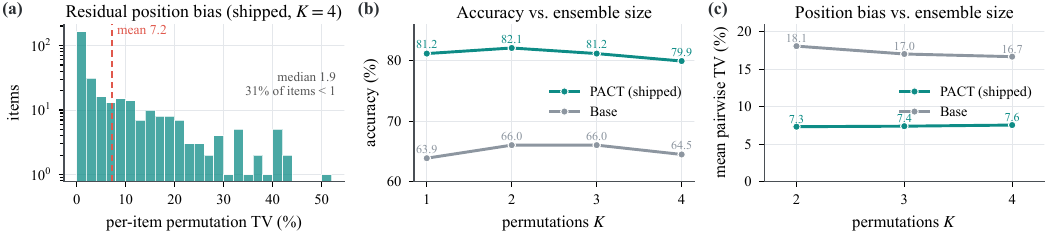}
\caption{Position bias. (a) Distribution of per-item permutation TV for the
shipped PACT model under $K{=}4$ (log scale). (b, c) Local 4-bit sweep over the
ensemble size $K$ for the shipped model and the base model: accuracy (b) and
mean pairwise TV between orderings (c).}
\label{fig:posbias}
\end{figure*}

\begin{figure*}[t]
\centering
\includegraphics[width=\textwidth]{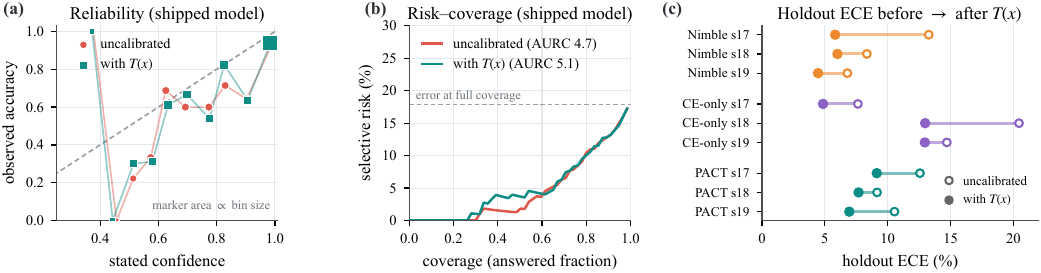}
\caption{Calibration. (a) Reliability diagram of the shipped model before and
after $T(x)$; marker area grows with bin size. (b) Risk--coverage curves,
items ranked by confidence. (c) Holdout ECE before (hollow) and after (filled)
$T(x)$ for every run and seed.}
\label{fig:calibration}
\end{figure*}

\subsection{Main comparison}
\label{sec:main}
Table~\ref{tab:main} and Fig.~\ref{fig:main} give the headline numbers, and
Table~\ref{tab:significance} the paired tests. Three findings stand out.

\emph{PACT matches, but does not beat, the published recipe in accuracy.}
PACT reaches 84.6$\pm$2.8\% against 85.2$\pm$2.1\% for \run{Nimble recipe}.
At every seed the difference is inside the bootstrap interval and McNemar
$p\ge0.50$; at two seeds the accuracies are identical. Properly reproduced with
checkpoint selection, the published recipe is a strong accuracy baseline, and
it is also the best-calibrated run (ECE$_T$ 5.4\% vs.\ 7.9\% for PACT).

\emph{The optimisation changes alone hurt; the PACT terms repair them.}
\run{CE-only}---PACT's optimiser and schedule without the four terms---is the
weakest and least stable fine-tuned arm (82.2$\pm$6.0\%, NLL 0.760, ECE
14.3\%). Adding the terms raises mean accuracy by 2.4\,pp, significantly so at
seeds 18 and 19 ($p=0.008$ and $0.035$; seed 17 favours \run{CE-only},
$p=0.099$), halves the seed-to-seed standard deviation (6.0$\to$2.8\,pp) and
lowers NLL by 26\% and ECE by 25\%. The terms thus act as regularisers that
make a more aggressive optimiser safe, rather than as a source of accuracy
beyond a well-tuned baseline.

\emph{Where PACT does lead, the lead matches the structure it encodes.}
PACT has the highest mean pair accuracy (73.3\%, not significant given the
6\,pp seed spread), the lowest Score MAE of all multi-seed arms
(0.232 vs.\ 0.311 and 0.322), and the lowest position bias
(Section~\ref{sec:posbias}). PACT clears the base model by 18--24\,pp
($p<0.001$ at every seed).

\begin{table}[t]
\centering
\caption{Paired tests on holdout accuracy (\%), same 324 items. $b/c$: items only PACT / only
the control gets right; exact two-sided McNemar $p$; 95\% CI of the accuracy difference from a
10{,}000-resample paired bootstrap (pp). \sigc{Highlighted}: $p<0.05$.
$^\dagger$Base-model predictions are from a local 4-bit re-evaluation (acc.\ 63.9\%).}
\label{tab:significance}
\setlength{\tabcolsep}{3.1pt}
\footnotesize
\rowcolors{2}{white}{rowgray}
\begin{tabular}{c ccc c c c}
\toprule
\rowcolor{hdr}
\hd{Seed} & \hd{PACT} & \hd{Ctrl.} & \hd{$\Delta$} & \hd{$b/c$} & \hd{$p$} & \hd{95\% CI} \\
\midrule
\multicolumn{7}{l}{\cellcolor{subhdr}\textit{PACT vs.\ Nimble recipe}} \\
17 & 84.0 & 84.0 & 0.0 & 19/19 & 1.000 & [$-$3.7, +3.7] \\
18 & 82.1 & 84.0 & $-$1.9 & 25/31 & 0.504 & [$-$6.5, +2.8] \\
19 & 87.7 & 87.7 & 0.0 & 27/27 & 1.000 & [$-$4.6, +4.3] \\
\midrule
\multicolumn{7}{l}{\cellcolor{subhdr}\textit{PACT vs.\ CE-only}} \\
17 & 84.0 & 87.3 & $-$3.4 & 13/24 & 0.099 & [$-$7.1, +0.3] \\
18 & 82.1 & 75.6 & +6.5 & 39/18 & \sigc{0.008} & [+1.9, +11.1] \\
19 & 87.7 & 83.6 & +4.0 & 23/10 & \sigc{0.035} & [+0.6, +7.4] \\
\midrule
\multicolumn{7}{l}{\cellcolor{subhdr}\textit{PACT vs.\ Base model$^\dagger$}} \\
17 & 84.0 & 63.9 & +20.1 & 92/27 & \sigc{$<$0.001} & [+13.9, +26.2] \\
18 & 82.1 & 63.9 & +18.2 & 86/27 & \sigc{$<$0.001} & [+12.0, +24.4] \\
19 & 87.7 & 63.9 & +23.8 & 98/21 & \sigc{$<$0.001} & [+17.9, +29.9] \\
\bottomrule
\end{tabular}
\end{table}

\begin{table}[t]
\centering
\caption{Permutation-ensembled decoding ($K{=}4$), \%. TV: mean pairwise total variation
between the $K$ canonical distributions; Flip: share of items whose argmax changes across
orderings. $\Delta_K$: accuracy gain of $K{=}4$ over $K{=}1$ (pp); a position-invariant
model should gain nothing.}
\label{tab:ensemble}
\setlength{\tabcolsep}{2.6pt}
\footnotesize
\rowcolors{2}{white}{rowgray}
\begin{tabular}{l ccccc c}
\toprule
\rowcolor{hdr}
\hd{Run} & \hd{Acc.} & \hd{Pair} & \hd{ECE} & \hd{TV\,$\downarrow$} & \hd{Flip\,$\downarrow$} & \hd{$\Delta_K$} \\
\midrule
Base model & 67.6 & 38.3 & 21.4 & 12.0 & 23.5 & \gainc{+1.2} \\
Nimble recipe & \best{85.7\,{\scriptsize$\pm$\,2.6}} & 72.4\,{\scriptsize$\pm$\,4.4} & \best{4.7\,{\scriptsize$\pm$\,1.0}} & 6.9\,{\scriptsize$\pm$\,0.9} & 13.8\,{\scriptsize$\pm$\,1.8} & \gainc{+0.5} \\
CE-only & 84.6\,{\scriptsize$\pm$\,4.6} & \best{73.3\,{\scriptsize$\pm$\,6.8}} & 8.2\,{\scriptsize$\pm$\,3.4} & 7.4\,{\scriptsize$\pm$\,1.6} & 12.1\,{\scriptsize$\pm$\,3.8} & \gainc{+2.4} \\
\rowcolor{pactrow}\textbf{PACT (ours)} & 83.8\,{\scriptsize$\pm$\,2.6} & 72.4\,{\scriptsize$\pm$\,5.5} & 8.6\,{\scriptsize$\pm$\,1.6} & \best{5.6\,{\scriptsize$\pm$\,1.4}} & \best{9.8\,{\scriptsize$\pm$\,2.1}} & \gainc{$-$0.7} \\
\bottomrule
\end{tabular}
\end{table}

\begin{table*}[t]
\centering
\caption{Seed-matched leave-one-out ablations (seed 17 for every row; single-pass except the
last two columns, which use $K{=}4$). Cells are coloured by the change relative to the full
method in the first row (seed 17):
\swatch{hurtB}\,removing the term clearly hurts, \swatch{hurtA}\,slightly hurts,
\swatch{helpA}\,slightly helps, \swatch{helpB}\,clearly helps; uncoloured: change below a
quarter of the ``clear'' threshold (2\,pp for accuracy-type metrics). One seed per row: read colours as
directions, not as significance.}
\label{tab:ablation}
\setlength{\tabcolsep}{5.6pt}
\begin{tabular}{l cccccc cc}
\toprule
\rowcolor{hdr}
\hd{Variant (seed 17)} & \hd{Acc.\,$\uparrow$} & \hd{Pair acc.\,$\uparrow$} & \hd{NLL\,$\downarrow$} &
\hd{ECE\,$\downarrow$} & \hd{AURC\,$\downarrow$} & \hd{Score MAE\,$\downarrow$} &
\hd{TV$_{K=4}$\,$\downarrow$} & \hd{Flip$_{K=4}$\,$\downarrow$} \\
\midrule
\rowcolor{pactrow}\textbf{PACT (ours)} & 84.0 & 72.8 & 0.659 & 12.6 & 4.44 & 0.229 & 5.0 & 8.0 \\
\midrule
$-\mathcal{L}_{\mathrm{CF}}$ & 84.3 & \cellcolor{hurtA}71.6 & \cellcolor{helpB}0.491 & \cellcolor{helpA}11.1 & \cellcolor{helpB}2.74 & \cellcolor{hurtB}0.284 & 5.3 & \cellcolor{hurtA}9.6 \\
$-\mathcal{L}_{\mathrm{PC}}$ & \cellcolor{helpB}87.7 & \cellcolor{helpB}79.6 & \cellcolor{helpB}0.513 & \cellcolor{helpB}7.5 & \cellcolor{helpB}2.82 & \cellcolor{helpB}0.138 & 4.8 & \cellcolor{hurtA}9.3 \\
$-\mathcal{L}_{\mathrm{NR}}$ & \cellcolor{hurtB}80.9 & \cellcolor{hurtB}64.8 & \cellcolor{helpB}0.545 & \cellcolor{helpB}10.3 & \cellcolor{hurtB}6.63 & \cellcolor{hurtB}0.344 & \cellcolor{hurtB}11.7 & \cellcolor{hurtB}17.6 \\
$-\mathcal{L}_{\mathrm{EMD}}$ & 84.3 & \cellcolor{hurtB}70.4 & \cellcolor{helpB}0.388 & \cellcolor{helpB}5.8 & \cellcolor{helpB}3.43 & \cellcolor{hurtB}0.282 & \cellcolor{hurtB}8.9 & \cellcolor{hurtB}15.7 \\
$-$permuted views & \cellcolor{helpA}84.9 & \cellcolor{helpA}74.1 & \cellcolor{helpB}0.510 & \cellcolor{helpB}8.8 & \cellcolor{helpA}3.51 & \cellcolor{hurtB}0.356 & \cellcolor{hurtB}6.6 & \cellcolor{hurtB}13.0 \\
\midrule
CE-only (all four off) & \cellcolor{helpB}87.3 & \cellcolor{helpB}75.9 & \cellcolor{helpB}0.454 & \cellcolor{helpB}7.6 & \cellcolor{helpB}2.66 & \cellcolor{helpA}0.182 & \cellcolor{hurtB}9.0 & \cellcolor{hurtB}16.4 \\
\bottomrule
\end{tabular}
\end{table*}

\begin{table*}[t]
\centering
\caption{Pre-specified hypotheses (Section~\ref{sec:hypotheses}) against the evidence.
``Seed-matched'' values compare against PACT at seed 17; multi-seed values are means over three seeds.}
\label{tab:scorecard}
\footnotesize
\setlength{\tabcolsep}{4pt}
\rowcolors{2}{white}{rowgray}
\begin{tabular}{c >{\raggedright\arraybackslash}p{3.6cm} >{\raggedright\arraybackslash}p{9.6cm} c}
\toprule
\rowcolor{hdr}
\hd{H} & \hd{Prediction} & \hd{Evidence} & \hd{Verdict} \\
\midrule
H1 & $-\Lcf$ loses more pair accuracy than accuracy &
Seed-matched, $K{=}1$: accuracy $+0.3$\,pp, pair accuracy $-1.2$\,pp (direction as predicted);
under $K{=}4$ the direction reverses ($+1.5$ / $+2.5$\,pp). Mechanism confirmed: $d$ rises from $-0.3$ to $>10$ nats (Fig.~\ref{fig:training}). &
\verdict{vMixed}{INCONCLUSIVE} \\
H2 & Lower TV/flip than CE-only; smaller $K{=}4$ gain &
TV 5.6 vs.\ 7.4\%, flip 9.8 vs.\ 12.1\%, $\Delta_K$ $-0.7$ vs.\ $+2.4$\,pp; also below the Nimble recipe.
Seed-matched: dropping the JSD term raises flips 8.0$\to$9.3\%, dropping all permuted views 8.0$\to$13.0\%. &
\verdict{vGood}{SUPPORTED} \\
H3 & $-\Lnr$ more confident on ablated contexts; no better pair accuracy &
Confidence on ablated contexts was not logged for $-\Lnr$ (untested half). Pair accuracy $-8.0$\,pp,
accuracy $-3.1$\,pp, TV 5.0$\to$11.7\%; with $\Lnr$ the ablated-context gap shrinks from 3.4 to $<$0.1 nats. &
\verdict{vPart}{PARTLY TESTED} \\
H4 & $-\Lemd$ raises Score MAE at similar Score accuracy &
Score MAE 0.229$\to$0.282, but Score accuracy is \emph{not} similar (79.7$\to$84.4\%). Multi-seed: lowest Score MAE
of all arms; Score accuracy above the Nimble recipe at all three seeds. &
\verdict{vPart}{PARTLY SUPPORTED} \\
H5 & ECE falls under $T(x)$, accuracy kept &
Holdout ECE lower in 9/9 runs (PACT 10.8$\to$7.9\%); accuracy unchanged by construction; AURC slightly worse (3.95$\to$4.21). &
\verdict{vGood}{SUPPORTED} \\
\bottomrule
\end{tabular}
\end{table*}

\subsection{Position bias and permutation ensembling}
\label{sec:posbias}
Table~\ref{tab:ensemble} reports $K{=}4$ decoding. PACT has the lowest TV
(5.6\% vs.\ 6.9\% for \run{Nimble recipe}, 7.4\% for \run{CE-only} and 12.0\%
for the base model) and the lowest flip rate (9.8\% vs.\ 13.8\%, 12.1\% and
23.5\%). It is also the only arm that gains nothing from ensembling
($\Delta_K=-0.7$\,pp vs.\ $+2.4$\,pp for \run{CE-only}): its single pass
already captures what the ensemble would recover, which is precisely the
behaviour H2 predicts and the property that matters for single-pass serving.

Fig.~\ref{fig:posbias} looks closer. The bias of the shipped model is
concentrated: the median item has TV of 1.9\% and 31\% of items are below 1\%,
while a small tail of items carries TV above 30\%. A local 4-bit sweep over
$K\in\{1,\dots,4\}$ (single seed; supplementary) shows the residual TV of PACT
plateauing at 7.3--7.6\% against 16.7--18.1\% for the base model,
2.2--2.4$\times$ lower at every $K$, and accuracy peaking at $K{=}2$ for both
models. The non-monotone accuracy curve is not a quantisation artefact: the
BF16 multi-seed results show the same pattern (84.6\% at $K{=}1$, 83.8\% at
$K{=}4$).

\subsection{Calibration}
\label{sec:calibration}
The contextual temperature lowers holdout ECE in all nine multi-seed runs
(Fig.~\ref{fig:calibration}c; per-run values in Table~\ref{tab:calibration},
Appendix~\ref{app:tables}), on average from 10.8\% to 7.9\% for PACT, from
9.5\% to 5.4\% for \run{Nimble recipe} and from 14.3\% to 10.3\% for
\run{CE-only}; accuracy is unchanged by construction. For the shipped model the
fitted slopes are negative for both the number of choices ($b=-0.42$) and the
prompt length ($c=-0.48$), i.e.\ fields with many choices and long prompts are
\emph{under}-confident relative to short binary ones, which a single scalar
temperature cannot express (Fig.~\ref{fig:temperature}a).

Three caveats are reported rather than omitted. First, the \emph{shape} of the
fitted temperature is not stable across runs (Fig.~\ref{fig:temperature}b):
over the nine multi-seed runs the slope on $\log C$ ranges from $-1.60$ to
$+2.73$ and the slope on $\log(L/1000)$ from $-3.49$ to $+3.58$. With 226
calibration records the two slopes are weakly identified; what transfers to the
holdout is most plausibly the overall \emph{level} of the temperature, so a scalar
temperature is a reasonable default and the contextual form should be
regularised or fitted on a larger split. Second, on the calibration split of
the shipped seed the fit leaves ECE essentially unchanged (5.5\%$\to$5.7\%);
the gain appears only on the holdout. Third, because $T(x)$ re-ranks
confidences across items, it slightly worsens selective risk (AURC
3.95$\to$4.21 for PACT; Fig.~\ref{fig:calibration}b).

\begin{figure}[t]
\centering
\includegraphics[width=\columnwidth]{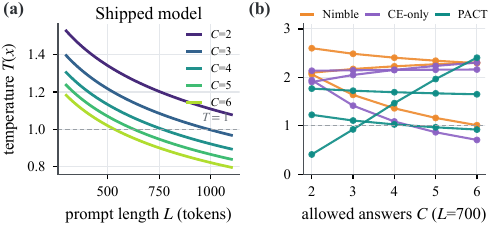}
\caption{Fitted contextual temperature. (a) Shipped model: $T(x)$ against prompt
length for each number of allowed answers $C$ in the holdout. (b) $T$ against $C$
at $L=700$ for every multi-seed run: for most runs $T>1$ (raw logits are
over-confident), but the slopes differ in sign between seeds.}
\label{fig:temperature}
\end{figure}

\subsection{Seed-matched ablations}
\label{sec:ablation}
Table~\ref{tab:ablation} compares each leave-one-out variant with PACT
\emph{at the same seed}. Two caveats frame the reading. Each row is a single
run, and the seed-17 PACT reference is itself a relatively weak draw
(84.0\% accuracy vs.\ 84.6\% mean, and the highest ECE of its three seeds), so
colours indicate directions, not significance. With that in mind, the pattern
is consistent:
\begin{itemize}
\item \textbf{No single term raises raw accuracy.} Removing $\Lcf$ or $\Lemd$
  leaves accuracy unchanged (84.3\% vs.\ 84.0\%); removing $\Lpc$ even raises
  it (87.7\%), as does removing all four terms (\run{CE-only}, 87.3\% at this
  seed, though not at the other two). Every removal also improves uncalibrated
  NLL and ECE, i.e.\ at this seed the regularisers trade
  calibration of the raw logits for other properties---a cost the post-hoc
  temperature largely absorbs (Section~\ref{sec:calibration}).
\item \textbf{The robustness properties are carried by the terms.} Removing
  $\Lnr$, $\Lemd$, the permuted views or all four terms raises the flip rate
  (8.0\% $\to$ 13.0--17.6\%) and TV; removing $\Lcf$, $\Lnr$, $\Lemd$ or the
  permuted views raises Score MAE (0.229 $\to$ 0.282--0.356).
\item \textbf{The necessity term is the most load-bearing.} Removing $\Lnr$
  degrades accuracy ($-$3.1\,pp), pair accuracy ($-$8.0\,pp), AURC
  (4.44$\to$6.63), Score MAE and position bias (TV 5.0$\to$11.7\%) at once.
\end{itemize}

\subsection{Where the differences are}
\label{sec:where}
Pooled accuracy hides a structured, seed-consistent pattern
(Fig.~\ref{fig:domains}). Relative to \run{Nimble recipe}, PACT is better on
Score fields at all three seeds (+6.2 to +15.6\,pp) and worse on Choice
($-$3.4 to $-$4.8\,pp) and Boolean fields ($-$0.9 to $-$4.4\,pp) at all three
seeds; per-kind accuracies are in Table~\ref{tab:kind}
(Appendix~\ref{app:tables}). The Score advantage is what the ordinal term
targets. By domain, PACT is better on \emph{public services} and worse on
\emph{education} and \emph{media} at every seed. We had no pre-specified
hypothesis about domains and report this as an observation for future work:
the gains of a structure-derived objective are heterogeneous, and a pooled
accuracy delta near zero can conceal offsetting per-kind effects of 5--15\,pp.

The same split appears in calibration (Fig.~\ref{fig:errors}c). Averaged over
seeds, PACT halves the ECE of Score fields relative to \run{Nimble recipe}
(10.8\% vs.\ 21.2\%; 10.1\% vs.\ 16.6\% after $T(x)$) but is worse on Choice
fields (16.0\% vs.\ 10.3\%), which is where its overall calibration deficit
comes from. (Per-kind ECE is computed on 64--146 items and is biased upwards
relative to the pooled value.)

\textbf{How pairs fail.} Fig.~\ref{fig:errors}b breaks the shipped model's 162
holdout pairs down by outcome. Both members are right in 67\% of pairs. When a
pair is only half right, it is usually the counterfactual that is missed
(35 pairs vs.\ 13 in which only the base is missed; for Boolean fields all 11
half-right pairs miss the counterfactual). Most strikingly, in 39 pairs (24\%;
33\% for Choice fields) the model gives the \emph{same} answer to both members,
i.e.\ it ignores the edited fact altogether---74\% of all imperfect pairs. This
is exactly the failure $\Lcf$ targets, and it remains the dominant error mode
after training: the mean margin on training pairs ends far above $m$
(Fig.~\ref{fig:training}b), but the behaviour does not fully transfer to new
source families.
Errors are also not always hesitant (Fig.~\ref{fig:errors}a): the median wrong
answer has confidence 0.74 against 0.995 for right answers, yet 16 of the 58
errors (28\%) are made with confidence $\ge0.9$, against 56 of 117 (48\%) for
the base model.

\begin{figure*}[t]
\centering
\includegraphics[width=\textwidth]{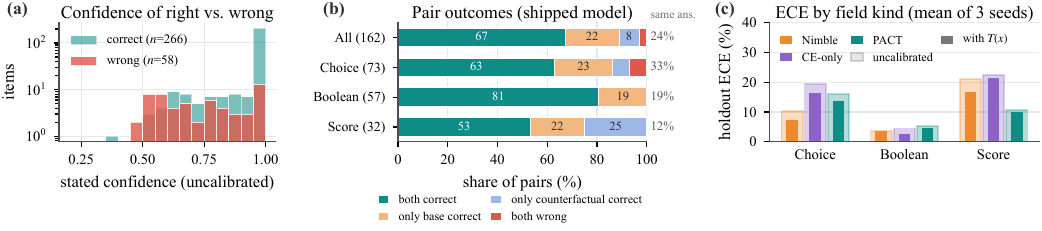}
\caption{Error analysis on the holdout. (a) Uncalibrated confidence of correct and
wrong answers of the shipped model (log scale). (b) Outcomes of the 162 contrastive
pairs by field kind (numbers: \% of pairs; right: share of pairs where both members
receive the same answer, which guarantees at least one error). (c) ECE by field kind,
mean over three seeds: wide light bars uncalibrated, narrow dark bars after $T(x)$.}
\label{fig:errors}
\end{figure*}

\begin{figure}[t]
\centering
\includegraphics[width=\columnwidth]{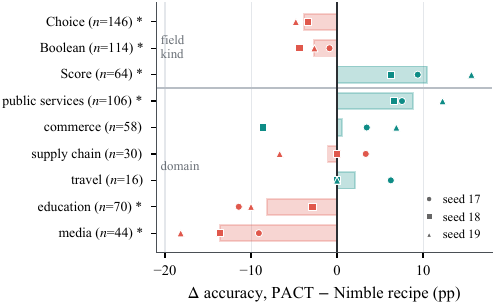}
\caption{Per-kind and per-domain accuracy difference, PACT minus
\run{Nimble recipe}, on the holdout. Markers: seeds; bars: mean. $^{*}$Same sign
at all three seeds.}
\label{fig:domains}
\end{figure}

\subsection{Training dynamics}
\label{sec:dynamics}
Fig.~\ref{fig:training} shows that the terms are optimised as intended. The
pair margin $d$ starts negative ($-0.3$ nats: the base model moves the log-odds
the wrong way), first exceeds the margin $m=2$ at step 20 and, after early
oscillation, ends near 10--13 nats. The logit gap on ablated contexts falls from 3.4 to below 0.1
nats, i.e.\ the model learns to be indifferent between $y_a$ and $y_b$ exactly
where the certificate says it should. $\Lnr$ is also the noisiest term, with
isolated spikes when a batch contains an ablated context the model still
separates strongly.

\begin{figure}[t]
\centering
\includegraphics[width=\columnwidth]{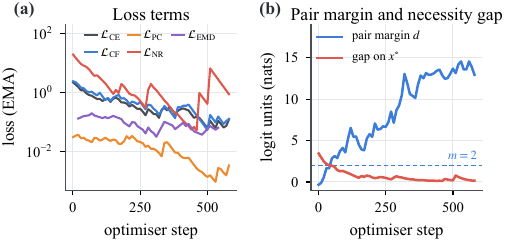}
\caption{Training dynamics of the shipped run (EMA-smoothed). (a) Loss terms.
(b) Mean pair margin $d$ of Eq.~\eqref{eq:did} (dashed: $m=2$) and mean gap
$|z_{y_a}-z_{y_b}|$ on ablated contexts $x^{\circ}$.}
\label{fig:training}
\end{figure}

\subsection{Hypothesis scorecard}
\label{sec:scorecard}
Table~\ref{tab:scorecard} confronts each pre-specified criterion with the
evidence. Two hypotheses are supported, two partly, and one is inconclusive;
none is contradicted outright, but H1---the motivating idea of the
paper---does not receive clean support from a single-seed ablation.

\section{Discussion}
\label{sec:discussion}

\textbf{What the evidence supports.} The results do not support the
proposition that structure-derived terms raise accuracy over a well-reproduced
cross-entropy baseline. They do support three narrower claims: (i) the terms
reduce answer-code position bias and ordinal error, the two failure modes they
were designed around; (ii) they stabilise an optimiser that is otherwise
significantly worse and highly seed-sensitive; and (iii) post-hoc
calibration of this model family is cheap and lowers ECE in every run, although
only the overall temperature level, not its contextual shape, is reliably
estimated. For a practitioner
who serves single-pass decisions over rubric fields, or who cannot afford $K$
forward passes to average out position bias, these are the properties that
matter; for a practitioner who only needs pooled accuracy, the published recipe
is sufficient and about five times cheaper to train.

\textbf{What remains open.} The dominant holdout failure is still the one the
pair margin was designed for: in a quarter of the pairs the model answers both
members identically, ignoring the edited fact (Section~\ref{sec:where}). The mean
margin is met on training pairs but the behaviour does not fully transfer to
unseen source families, which points to margin schedules, harder negative edits or
pair-level data augmentation as the most promising next steps.

\textbf{Cost.} Each pair contributes 2.5$\times$ the rows of the baseline
(Table~\ref{tab:cost}). Measured wall-clock per run is 1.37\,h for PACT against
0.50\,h for \run{CE-only} and 0.26\,h for the one-epoch \run{Nimble recipe}, and
peak GPU memory is 28.8\,GiB against 21.8\,GiB. Cost scales with the number of
views rather than with the loss terms: dropping $\Lpc$ alone saves nothing,
while dropping the permuted views saves 45\% of the time. Serving cost is
unchanged.

\begin{table}[t]
\centering
\caption{Training cost per run (one 48\,GB GPU per job): optimiser steps, training views per pair,
mean wall-clock time, peak GPU memory, and mean holdout accuracy (\%). Serving cost is identical
for all runs.}
\label{tab:cost}
\setlength{\tabcolsep}{3.2pt}
\footnotesize
\rowcolors{2}{white}{rowgray}
\begin{tabular}{l c c c c c}
\toprule
\rowcolor{hdr}
\hd{Run} & \hd{Steps} & \hd{Views/pair} & \hd{Hours} & \hd{Peak GiB} & \hd{Acc.} \\
\midrule
Nimble recipe & 292 & 2.0 & 0.26 & 21.8 & 85.2 \\
CE-only & 584 & 2.0 & 0.50 & 21.8 & 82.2 \\
\rowcolor{pactrow}\textbf{PACT} & 584 & 5.0 & 1.37 & 28.8 & 84.6 \\
\midrule
$-\mathcal{L}_{\mathrm{CF}}$ & 584 & 5.0 & 1.37 & 28.7 & 84.3 \\
$-\mathcal{L}_{\mathrm{PC}}$ & 584 & 5.0 & 1.36 & 28.7 & 87.7 \\
$-\mathcal{L}_{\mathrm{NR}}$ & 584 & 4.0 & 1.06 & 26.4 & 80.9 \\
$-\mathcal{L}_{\mathrm{EMD}}$ & 584 & 5.0 & 1.36 & 28.7 & 84.3 \\
$-$permuted views & 584 & 3.0 & 0.76 & 24.0 & 84.9 \\
\bottomrule
\end{tabular}
\end{table}

\textbf{Selection noise.} Checkpoint selection is itself a source of variance
comparable to the effects studied here (Fig.~\ref{fig:selection}). The shipped
model is the PACT seed with the lowest selection-split NLL (seed 18), yet it has
the lowest holdout accuracy of the three PACT seeds (82.1\%). More strikingly,
the checkpoint with the lowest selection NLL of \emph{all} nine runs
(\run{CE-only} seed 18: NLL 0.112, 97.4\% selection accuracy) has the
\emph{lowest} holdout accuracy of all fine-tuned runs (75.6\%). With only two
source families and 116 records, the selection split rewards fitting those
families rather than transfer, which is one more reason to report seeds
individually and to test differences in pairs.

\begin{figure}[t]
\centering
\includegraphics[width=\columnwidth]{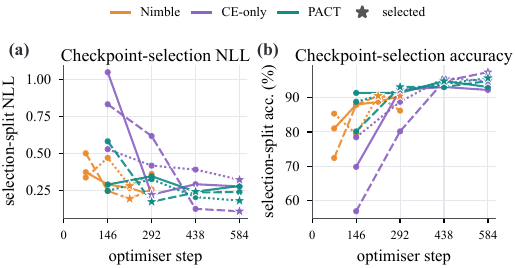}
\caption{Selection-split trajectories of the three multi-seed arms (line style:
seed 17 solid, 18 dashed, 19 dotted; stars: selected checkpoints). \run{CE-only}
converges slowly and reaches the best selection scores, which do not carry over
to the holdout (Table~\ref{tab:significance}).}
\label{fig:selection}
\end{figure}

\section{Limitations}
\label{sec:limits}
The holdout contains 324 synthetic, model-labelled items from six source
families. It is narrow, and its labels come from the same family of models as
the training labels, so shared errors are possible and undetectable from inside
the benchmark. $\Lnr$ and $\Lcf$ trust the certificate: a wrong necessity check
teaches indifference where the evidence actually decides. We apply the
necessity term to one ablated view per pair to bound cost. The ablations use one
seed each, so per-term claims are directional; the significance tests cover only
the three multi-seed arms. The loss weights were fixed a priori and not tuned,
so the negative accuracy findings are for this setting, not for every weighting.
All claims concern one base model at one scale with one adapter rank.

\section{Conclusion}
\label{sec:conclusion}
Contrastive curation with certificates contains supervision that cross-entropy
cannot express: the offset-free effect of an edited fact, the arbitrariness of
answer codes, the indifference certified by evidence removal, and the order of
rubric levels. PACT turns each into a training term without changing the
serving contract. Evaluated under pre-specified falsification criteria, the
terms do not make the model more accurate than a well-reproduced baseline, but
they make it less sensitive to answer-code order, better on ordinal fields and
more stable across seeds, and a three-parameter temperature makes its
probabilities usable. We hope the protocol---a reproduced baseline, a
matched-optimiser control, seed-matched ablations and paired tests---is as
useful as the method.

\bibliographystyle{IEEEtran}
\bibliography{refs}

\appendices

\section{Additional Tables}
\label{app:tables}
Table~\ref{tab:kind} breaks accuracy down by field kind for every run and
seed; Table~\ref{tab:calibration} gives the fitted temperatures and ECE before
and after calibration on both the calibration split and the holdout.

\begin{table}[t]
\centering
\caption{Holdout accuracy (\%) by field kind, single-pass decoding (Choice $n{=}146$,
Boolean $n{=}114$, Score $n{=}64$), and expected-level MAE on Score fields.}
\label{tab:kind}
\setlength{\tabcolsep}{4pt}
\footnotesize
\rowcolors{2}{white}{rowgray}
\begin{tabular}{l c ccc c}
\toprule
\rowcolor{hdr}
\hd{Run} & \hd{Seed} & \hd{Choice} & \hd{Boolean} & \hd{Score} & \hd{Score MAE} \\
\midrule
Base model & -- & 59.6 & 84.2 & 50.0 & 0.571 \\
\midrule
Nimble recipe & 17 & 80.8 & 95.6 & 70.3 & 0.326 \\
Nimble recipe & 18 & 81.5 & 94.7 & 70.3 & 0.331 \\
Nimble recipe & 19 & 84.9 & 97.4 & 76.6 & 0.277 \\
\midrule
CE-only & 17 & 78.8 & 97.4 & 89.1 & 0.182 \\
CE-only & 18 & 68.5 & 94.7 & 57.8 & 0.486 \\
CE-only & 19 & 78.1 & 94.7 & 76.6 & 0.299 \\
\midrule
\rowcolor{pactrow}\textbf{PACT} & 17 & 77.4 & 94.7 & 79.7 & 0.229 \\
\rowcolor{pactrow}\textbf{PACT} & 18 & 78.1 & 90.4 & 76.6 & 0.353 \\
\rowcolor{pactrow}\textbf{PACT} & 19 & 80.1 & 94.7 & 92.2 & 0.114 \\
\midrule
$-\mathcal{L}_{\mathrm{CF}}$ & 17 & 79.5 & 95.6 & 75.0 & 0.284 \\
$-\mathcal{L}_{\mathrm{PC}}$ & 17 & 80.1 & 94.7 & 92.2 & 0.138 \\
$-\mathcal{L}_{\mathrm{NR}}$ & 17 & 76.7 & 93.9 & 67.2 & 0.344 \\
$-\mathcal{L}_{\mathrm{EMD}}$ & 17 & 75.3 & 95.6 & 84.4 & 0.282 \\
$-$permuted views & 17 & 83.6 & 94.7 & 70.3 & 0.356 \\
\bottomrule
\end{tabular}
\end{table}

\begin{table}[t]
\centering
\caption{Contextual temperature per run: mean [min, max] of $T(x)$ on the calibration split,
and ECE (\%) before/after on the calibration split (fit) and on the holdout (test).
$^\star$Shipped model. \worse{Red}: $T(x)$ raised ECE.}
\label{tab:calibration}
\setlength{\tabcolsep}{3pt}
\footnotesize
\rowcolors{2}{white}{rowgray}
\begin{tabular}{l c c cc cc}
\toprule
\rowcolor{hdr}
 & & & \multicolumn{2}{c}{\hd{Calib.\ split}} & \multicolumn{2}{c}{\hd{Holdout}} \\
\rowcolor{hdr}
\hd{Run} & \hd{Seed} & \hd{$T(x)$} & \hd{pre} & \hd{post} & \hd{pre} & \hd{post} \\
\midrule
Nimble recipe & 17 & 1.85 [1.10, 2.89] & 14.4 & 6.5 & 13.3 & 5.8 \\
Nimble recipe & 18 & 1.54 [0.99, 2.00] & 11.4 & 8.3 & 8.3 & 6.0 \\
Nimble recipe & 19 & 1.85 [1.28, 2.41] & 16.1 & 9.1 & 6.8 & 4.5 \\
\midrule
CE-only & 17 & 2.01 [1.78, 2.23] & 15.5 & 6.0 & 7.6 & 4.9 \\
CE-only & 18 & 2.12 [1.95, 2.32] & 15.7 & 4.5 & 20.4 & 13.0 \\
CE-only & 19 & 1.56 [0.70, 2.45] & 8.1 & 7.6 & 14.7 & 13.0 \\
\midrule
\rowcolor{pactrow}\textbf{PACT} & 17 & 1.56 [0.69, 2.69] & 12.8 & 5.7 & 12.6 & 9.1 \\
\rowcolor{pactrow}\textbf{PACT} & 18$^\star$ & 1.16 [0.92, 1.40] & 5.5 & \worse{5.7} & 9.2 & 7.7 \\
\rowcolor{pactrow}\textbf{PACT} & 19 & 1.66 [1.61, 1.76] & 12.5 & 6.0 & 10.5 & 6.9 \\
\bottomrule
\end{tabular}
\end{table}

\section{Implementation and Reproducibility}
\label{app:impl}
All code, configurations, data splits, trained adapters, run records and the
scripts that regenerate every figure and table of this paper are released at
\repourl. The training driver records, per run, the resolved configuration, the data
audit and file hashes, the view-cache fingerprint, library versions, device
information, the training plan, the full step log, the selection-split history,
the fitted calibrator and every per-item holdout prediction with its
probabilities. Adapters are saved with the recipe's \texttt{schema\_config.json}
contract, including the SHA-256 of the scoring-prompt implementation, so that a
checkpoint is rejected if the prompt code it was trained against has changed.
All figures and data tables in this paper are regenerated from these records by
a single script. A 4-bit NF4 configuration for 24\,GB GPUs is provided but its
results are not mixed with the BF16 numbers reported here; the only 4-bit
numbers in the paper are the supplementary $K$-sweep and the base-model
predictions used for the paired tests. Table~\ref{tab:hparams} lists every
hyper-parameter; none was tuned on the holdout.

\begin{table}[t]
\centering
\caption{Hyper-parameters of PACT (read from the resolved run configuration). The
Nimble recipe differs only in: 1 epoch, linear schedule, no rsLoRA, LoRA+ ratio 1, no extra
views or terms.}
\label{tab:hparams}
\setlength{\tabcolsep}{4pt}
\footnotesize
\rowcolors{2}{white}{rowgray}
\begin{tabular}{l >{\raggedright\arraybackslash}p{5.2cm}}
\toprule
\rowcolor{hdr}
\hd{Setting} & \hd{Value} \\
\midrule
Base model & Qwen3.5-9B, BF16, SDPA attention \\
Max. prompt length & 2,048 tokens (longest seen: 1,098) \\
LoRA & rank 16, $\alpha$=32, dropout 0.05, rsLoRA \\
LoRA targets & q,k,v,o,gate,up,down (language model) \\
Optimiser & AdamW, lr $5\times10^{-5}$, weight decay 0 \\
LoRA+ ratio & 8 (lr of $B$ / lr of $A$) \\
Schedule & cosine, warm-up 10\%, 2 epochs \\
Batch & 2 pairs $\times$ 2 accumulation = 8 primary rows \\
Gradient clipping & global norm 1.0 \\
Loss weights & $\lambda_{\mathrm{cf}}$=0.5, $\lambda_{\mathrm{pc}}$=0.5, $\lambda_{\mathrm{nr}}$=0.2, $\lambda_{\mathrm{emd}}$=0.3 \\
Margins & $m$=2, $\varepsilon$=0; ramp 10\% of steps \\
Length bucketing & 64 groups per bucket \\
Selection & every 25\% of training, lowest selection NLL \\
Evaluation & $K$=4 permutations, ECE with 15 bins \\
Calibration & contextual $T(x)$, L-BFGS, 120 iterations \\
Seeds & 17, 18, 19 (ablations: 17) \\
\bottomrule
\end{tabular}
\end{table}

\section{Applied Demonstration: Sequential Decisions in Tetris}
\label{app:tetris}
Everything above evaluates the adapter on a static document plus a schema
question. As a qualitative check of how general the serving interface is, we
reused it unmodified for a task with no document at all: choosing where to
place falling pieces in a from-scratch Tetris implementation. At every piece
the game engine---not the model---enumerates the legal placements and computes
their consequences two plies deep (rows cleared, holes, stack height and
bumpiness, now and after the best placement of the visible next piece; standard
Tetris features~\cite{thiery2009tetris}). Each option becomes one choice
description in an ordinary schema call; the model reads one token and the
highest-logit option is played.

\textbf{Policy-gradient adaptation.} Unmodified, the adapter does not reliably
choose the option the two-ply search ranks best: it was trained to classify
documents, not to trade off sequential consequences. We therefore further
fine-tuned the same LoRA adapter with single-step REINFORCE~\cite{williams1992reinforce}
(a contextual bandit over candidate placements), using the two-ply score of the
sampled placement as reward and the mean score over that state's candidates as
a per-state baseline, with an entropy bonus of 0.01, learning rate
$10^{-5}$ and 32 moves per update on 2$\times$40\,GB GPUs. Prompt construction
is shared between training and serving, so there is no train/serve skew.

\textbf{Results.} The sampled top-pick rate---agreement with the search's
best-ranked option---rose from 19\% at the first update to 87.5\% by step
16, and its 100-step average stayed between 88\% and 99\% from step 200 to the
end of the 1{,}574-step run (Fig.~\ref{fig:rl}a). The training game itself
never topped out after step 14, clearing 19{,}975 lines over roughly 50{,}000
pieces. Greedy evaluation with a 400-piece cap (Fig.~\ref{fig:rl}b) is flat at
a mean of 17{,}000--18{,}000 points and about 158 lines, close to the ceiling of
160 lines that 400 pieces permit, so this protocol saturates and cannot show
improvement; one uncapped greedy game at the final checkpoint reached 1{,}500
pieces and 598 lines (score 65{,}700) before it was stopped.

\textbf{Live play in the served app.} The repository (Appendix~\ref{app:impl})
includes a three-minute
recording at 2$\times$ speed (\texttt{media/pact\_rl\_tetris\_2x.mp4}) of the
final checkpoint driving the unmodified browser app through the ordinary
serving path, in 4-bit on a single 16\,GB consumer GPU. In six minutes of real
play the model cleared more than 130 lines in one continuous game and reached
level 13 without topping out, although gravity accelerates with every level and
each move requires a fresh forward pass.

\textbf{Reading this appendix.} This is a single run without seeds, held-out
repetitions or a greedy pre-training baseline, and the policy learns to follow
a ranking that the engine already computes. It demonstrates only that the
single-token interface is a general typed-choice interface and that the same
lightweight adapter can be re-targeted to a new decision task; it carries no
evidence about the holdout claims above.

\begin{figure}[t]
\centering
\includegraphics[width=\columnwidth]{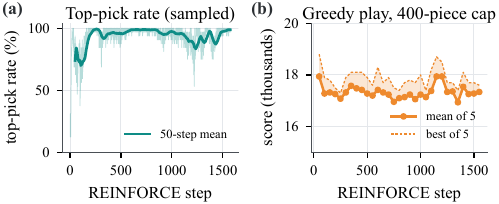}
\caption{REINFORCE on Tetris placement choices. (a) Share of sampled moves that
match the two-ply search's top-ranked option (faint: per update; bold: 50-step
mean). (b) Greedy evaluation every 50 steps, 5 games capped at 400 pieces
(line: mean; dashed: best game).}
\label{fig:rl}
\end{figure}

\end{document}